\documentclass[letterpaper, 10pt, conference]{ieeeconf}
\IEEEoverridecommandlockouts 
\usepackage{amsmath,amssymb,url}
\usepackage{graphicx,subfigure}
\usepackage{color}
\usepackage{siunitx}
\usepackage[hidelinks]{hyperref}
\usepackage{lipsum} 
\usepackage{cleveref}
\usepackage{booktabs}

\newcommand{\deriv}[2]{\ensuremath{\frac{\partial #1}{\partial #2}}}

\newcommand{\SO}{\ensuremath{\mathsf{SO(3)}}}

\newcommand{\so}{\ensuremath{\mathfrak{so}(3)}}

\renewcommand{\Re}{\ensuremath{\mathbb{R}}}
\newcommand{\Sph}{\ensuremath{\mathsf{S}}}

\title{\LARGE \bf
	Equivariant Reinforcement Learning for Quadrotor UAV
}
\author{Beomyeol Yu and Taeyoung Lee
	\thanks{Beomyeol Yu, Taeyoung Lee, Mechanical and Aerospace Engineering, George Washington University, Washington, DC 20051 {\tt \{yubeomyeol,tylee\}@gwu.edu}}%
	\thanks{\textsuperscript{\footnotesize\ensuremath{*}}This research has been supported in part by NSF under the grants CNS-1837382 and CMMI-1760928.}%
}

\newtheorem{prop}{Proposition}

\graphicspath{{./figs/}}

\makeatletter
\def\endthebibliography{%
	\def\@noitemerr{\@latex@warning{Empty `thebibliography' environment}}%
	\endlist
}
\makeatother

\begin{document}
	\allowdisplaybreaks
	\maketitle \thispagestyle{empty} \pagestyle{empty}
	
	\begin{abstract}
		This paper presents an equivariant reinforcement learning framework for quadrotor unmanned aerial vehicles. 
		Successful training of reinforcement learning often requires numerous interactions with the environments, which hinders its applicability especially when the available computational resources are limited, or when there is no reliable simulation model.
		We identified an equivariance property of the quadrotor dynamics such that the dimension of the state required in the training is reduced by one, thereby improving the sampling efficiency of reinforcement learning substantially. 
		This is illustrated by numerical examples with popular reinforcement learning techniques of TD3 and SAC.
	\end{abstract}

	\section{Introduction}
	
	Deep reinforcement learning (DRL) algorithms have been successfully applied to optimal control of complex dynamic systems through an interactive learning process.
	They have made remarkable advancements in various applications such as games \cite{mnih2015human} or natural language processing \cite{choi2017coarse}.
	Recently, their application domain has been extended to robotics, including unmanned aerial systems, such as drone racing \cite{song2021autonomous} and payload transportation of quadrotors \cite{belkhale2021model}.
	Compared to traditional model-based control methods \cite{bouabdallah2004pid,xu2006sliding,lee2010geometric}, DRL-based control does not require any exact mathematical model to achieve its goals.
	
	Prior works in reinforcement learning of quadrotors have mainly focused on introducing DRL-based control strategies for hovering and trajectory tracking tasks~\cite{pi2020low, hwangbo2017control}.
	For example, to enhance tracking accuracy, a stochastic policy has been trained in \cite{lopes2018intelligent} with Proximal Policy Optimization (PPO) for an on-policy method, and Twin Delayed Deep Deterministic Policy Gradient (TD3) has been adopted in \cite{shehab2021low} to develop a deterministic policy as an off-policy technique.
	In~\cite{rodriguez2019deep}, Deep Deterministic Policy Gradients (DDPG) have been used for autonomous landing on a moving platform, which has been validated by simulation and flight experiments.
	In the study of \cite{molchanov2019sim} and \cite{wang2019deterministic}, the authors have focused on mitigating the \textit{reality gap} that appears when transferring the policy trained in simulation into the real world, while greatly improving robustness.
	However, since DRL-based control is a data-driven approach relying on deep neural networks, most of their recent successes have faced the challenge of handling complex and high-dimensional data.
	This process often requires numerous samples for successful learning, thereby degrading its efficiency in both computation and learning. 
	
	\begin{figure}[t]
		\centering
		\begin{picture}(145,135)
		\put(0,0){\includegraphics[scale=0.3]{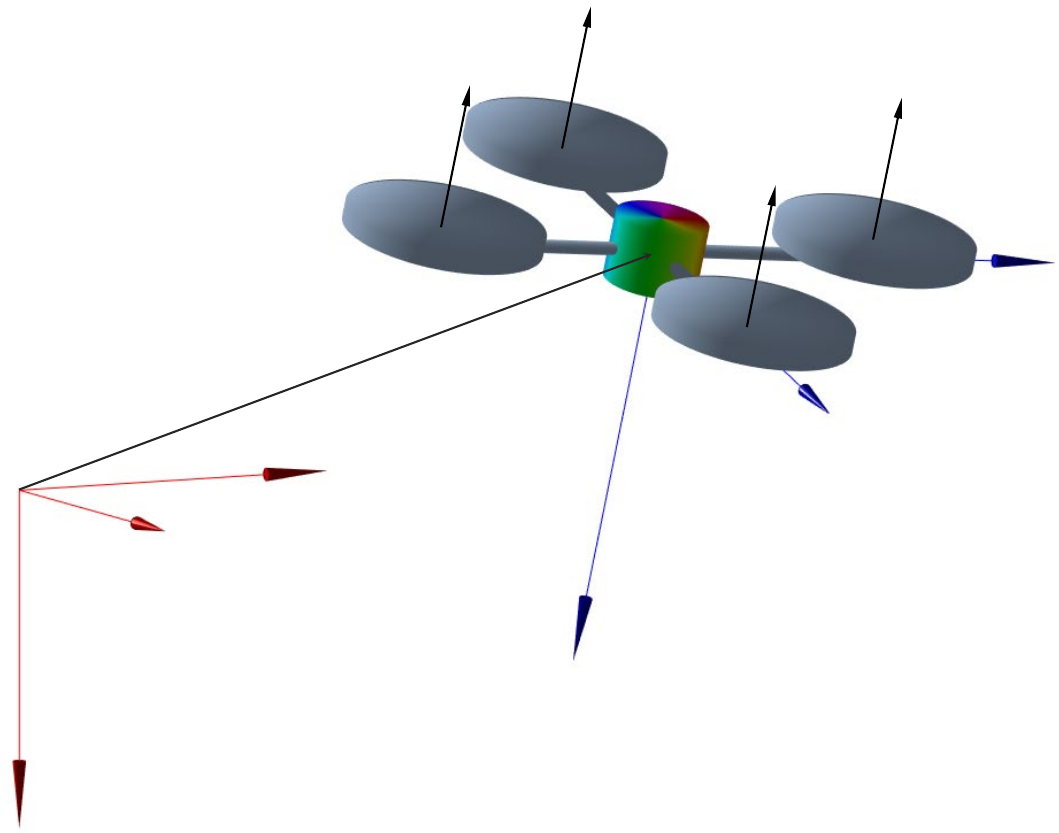}}
		\put(49,49){{\small $\vec e_1$}}
		\put(26,39){{\small $\vec e_2$}}
		\put(0,-7){{\small $\vec e_3$}}		
		\put(155,79){{\small $\vec b_1$}}
		\put(118,50){{\small $\vec b_2$}}
		\put(80,15){{\small $\vec b_3$}}
		\put(53,62){{\small $x\in\Re^3$}}
		\put(135,64){{\footnotesize $R\in\SO$}}
		\put(127,108){{\small $T_1$}}
		\put(109,96){{\small $T_2$}}
		\put(65,110){{\small $T_3$}}
		\put(82,121){{\small $T_4$}}
		\end{picture}
		\caption{Quadrotor Model}
		\label{fig:DM}
        \vspace*{-0.3cm}
	\end{figure}
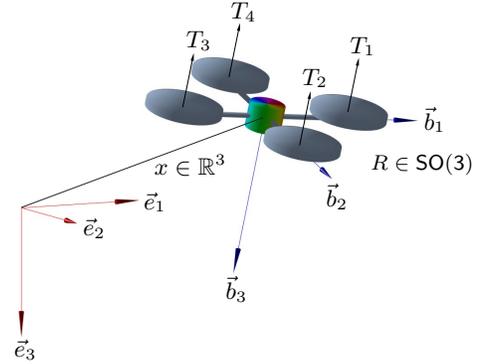
	
	A popular way to address this issue in deep learning is \textit{equivariant neural network}, which is part of a broad theme of geometric deep learning~\cite{bronstein2017geometric}.
	Equivariant models can improve sample efficiency and generalization capability by directly utilizing the geometric relationship between the input and output data, such as in translation, rotation, or permutation.
	The concept of equivariant learning was first proposed in~\cite{cohen2016group}, and it has been actively adopted in computer vision~\cite{weiler2019general}.
	In reinforcement learning, reflectional and rotational equivariance have been utilized in the formulation of homomorphic networks for discrete actions~\cite{van2020mdp,van2021multi}.
	Another approach has employed an equivariant architecture for vision-based robotic manipulation in Q-learning and actor-critic methods~\cite{wang2022mathrm}.	
		
	In this paper, we propose an equivariant framework of reinforcement learning for quadrotor low-level control, which directly maps the state of the quadrotor to motor control signals.
	Particularly, we identify a rotational symmetry, where the optimal control represented in the body-fixed frame is invariant under the rotation about the gravity direction. 
    By embedding this particular structure in an actor-critic architecture, the dimension of the sample trajectories used in the training is reduced by one, thereby improving the sampling efficiency substantially. 
	Data efficiency is particularly important in aerial robotics with a large-dimensional continuous state-action space.
	Further, as quadrotors are inherently unstable, it is critical to safely complete learning with a minimal number of trials. 
	We compare our agents trained with the proposed equivariant framework with non-equivariant counterparts in TD3 and SAC, to show significant computational advantages. 

	In short, our main contributions are two-fold.
	First, a data-efficient reinforcement learning scheme is proposed for quadrotors where the dimension of sample trajectories is reduced by one. 
    Second, it is shown that the proposed framework successfully improves the convergence of reinforcement learning through numerical simulations.

    \section{Problem Formulation}\label{sec:PF}
	We are interested in solving the problem of quadrotor low-level control. 
	This section provides a theoretical background in quadrotor dynamics and reinforcement learning.

	\subsection{Quadrotor Dynamics}\label{sec:DM}
	
	Consider a quadrotor unmanned aerial vehicle, illustrated at Fig. \ref{fig:DM}.
	Let $\{\vec e_1,\vec e_2,\vec e_3\}$ be the axes of the inertial frame, where the third axis $\vec e_3$ is aligned along the gravity pointing downward.
	And let $\{\vec b_{1},\vec b_{2},\vec b_{3}\}$ be the body-fixed frame located at the mass center of the quadrotor.
	The first two axes are aligned toward the center of the corresponding rotors, such that the third axis points downward when hovering.
	The position and the velocity of the quadrotor in the inertial frame are denoted by $x\in\Re^3$ and $v\in\Re^3$, respectively.
	The attitude is defined by the rotation matrix $R\in\SO=\{R\in\Re^{3\times 3}\,|\, R^T R=I_{3\times 3},\; \mathrm{det}[R]=1\}$, which is the linear transformation of the representation of a vector from the body-fixed frame to the inertial frame. 
	The angular velocity vector resolved in the body-fixed frame is $\Omega\in\Re^3$.
	
	The equations of motion for the quadrotor are given by 
	\begin{gather}
	\dot x  = v,\label{eqn:x_dot}\\
	m \dot v = mge_3 - f R e_3,\label{eqn:v_dot}\\
	\dot R = R\hat\Omega,\label{eqn:R_dot}\\
	J\dot \Omega + \Omega\times J\Omega = M,\label{eqn:W_dot}
	\end{gather}
	where the \textit{hat map} $\hat\cdot:\Re^3\rightarrow\so=\{S\in\Re^{3\times 3}\,|\, S^T = -S\}$ is defined by the condition that $\hat x y=x\times y$ and $\hat x$ is skew-symmetric for any $x,y\in\Re^3$.
	The inverse of the hat map is denoted by the \textit{vee map} $\vee:\so\rightarrow\Re^3$
	Also, $m\in\Re$ and $J\in\Re^{3\times 3}$ are the mass, and the inertia matrix of the quadrotor with respect to the body-fixed frame, respectively, and $g\in\Re$ is the gravitational acceleration.
	From the thrust of each motor denoted by $(T_1,T_2,T_3,T_4)$, the total thrust $f = \sum_{i=1}^{4} T_i \in\Re$ and the total moment $M \in\Re^3$ resolved in the body-fixed frame can be computed by
	\begin{gather}
	\begin{bmatrix} 
	f \\ M_1 \\ M_2 \\ M_3 
	\end{bmatrix}
	=
	\begin{bmatrix}
	1 & 1 & 1 & 1 \\
	0 & -d & 0 & d \\
	d & 0 & -d & 0 \\
	c_{\tau f} & -c_{\tau f} & c_{\tau f} & -c_{\tau f} 
	\end{bmatrix}
	\begin{bmatrix} 
	T_1 \\ T_2 \\ T_3 \\ T_4
	\end{bmatrix},
	\end{gather}
	where $d\in\Re$ is the distance between the center of any rotor and the third body-fixed axis, and $c_{\tau f} \in\Re$ is a constant relating the thrust and the resulting reactive torque.

	\subsection{Markov Decision Process} 
	Markov decision process (MDP) is an extension of Markov chains augmented by actions and rewards, which describe the choices available for each state and the objective to achieve.
	Specifically, it is defined by a tuple $\mathcal{M} = ( \mathcal{S}, \mathcal{A}, \mathcal{R}, \mathcal{T}, \gamma )$, where $\mathcal{S}$ is the state space, $\mathcal{A}$ is the action space, and $\mathcal{R}: \mathcal{S} \times \mathcal{A} \rightarrow \mathbb{R}$ is the reward function.
	Next, $\mathcal{T: \mathcal{S} \times \mathcal{A} \rightarrow \mathcal{P}(\mathcal{S})}$ denotes the state transition probability.
	For example, in a discrete-time setting, it is specified by $P(s_{t+1}|s_t, a_t)$, i.e., the distribution of the state at the next time step, for the given state and action at the current step.  
	At each time step $t$, the agent takes an action $a_t$ drawn from a policy $\pi(a_t|s_t)$, which is the distribution of the action conditioned by the state, receives a reward $r_t$, and a next state is determined by the transition probability function.
	This sequence of state-action pairs is called a trajectory or rollout $\tau = (s_0, a_0, s_1, a_1, \cdots, s_T,  a_T)$.
	The goal of MDP is to identify an optimal policy $\pi^*(a_k|s_k)$ that maximizes the expected return $J_t = \sum_{k=0}^{\infty} \gamma^k r_{t+k+1}$ where $\gamma \in [0, 1]$ is a discount factor.
    Reinforcement learning addresses MDP by constructing a policy iteratively while interacting with the dynamic system. 

	\subsection{Reinforcement Learning for Quadrotor}
	
	In this paper, the control objective is to find an optimal policy such that for a given desired position $x_d\in\Re^3$, $x\rightarrow x_d$ as $t\rightarrow\infty$.
    This is formulated as MDP as follows. 
    The state  and the action of the quadrotor are given by $s = (x, v, R, \Omega) \in \mathcal{S} = \Re^{9}\times \SO $ and $a = $$(T_{1}, T_{2}, T_{3}, T_{4})\in \mathcal{A} = \Re^4$, respectively. 
    The state transition probability is determined by the equations of motion \eqref{eqn:x_dot}--\eqref{eqn:R_dot}, which can be discretized according to a numerical integration scheme. 
	Next, to achieve the above stabilization objective, the reward function is defined as
	\begin{align}
        r_t (s_t, a_t) & = c_x (1 - \|e'_{x_t}\|) - c_v \|v_t\| - c_\Omega \|\Omega_t\| \nonumber \\
                   & \quad - c_a \|a_t - a_{t-1}\|. \label{eqn:reward}
	\end{align}
    where the constants $c$ are positive weighting factors, and $e'_{x_t} = (x_t - x_d)/e_{x_{\max}}\in\Re^3$ is the position error normalized such that $\|e'_x\|\leq 1$ always. 
    More specifically, it is assumed that the $i$-th element of $x$ is in the domain of $[x_{d_i} - e_{x_{\max}}, x_{d_i} + e_{x_{\max}}]$ with a prescribed $e_{x_{\max}}>0$ for $i\in\{1,2,3\}$, and any rollout is terminated once it is violated. 
    In \eqref{eqn:reward}, the first term is to minimize the scaled position error, and the next two terms are to mitigate aggressive motions. 
    The last term is to discourage chattering in control inputs, where it is considered that the prior action $a_{t-1}$ is prescribed at the $t$-th step with the convention of $a_{-1}=a_0$.
	
    The presented MDP for the quadrotor can be addressed by any reinforcement learning schemes that can handle continuous state spaces and action spaces, such as Twin Delayed Deep Deterministic Policy Gradient (TD3)~\cite{fujimoto2018addressing}, and Soft Actor-Critic (SAC)~\cite{haarnoja2018soft}.
	While these algorithms have been successfully applied in various challenging applications, they often require a massive amount of data to train their agents successfully.
    This motivates the proposed equivariant reinforcement learning as presented below. 

	\section{$\Sph^1$--Equivariant Reinforcement Learning for Quadrotor}

	In this section, we present a symmetry property of the quadrotor and we discuss how it yields an equivariance property to be utilized in enhancing the sampling efficiency of reinforcement learning.
	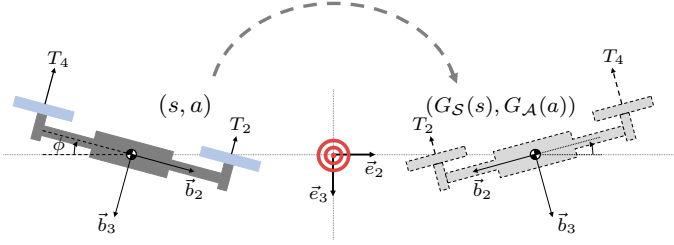
\begin{figure}[h]
		\begin{picture}(100,100)
		\put(0,0){\includegraphics[scale=0.32]{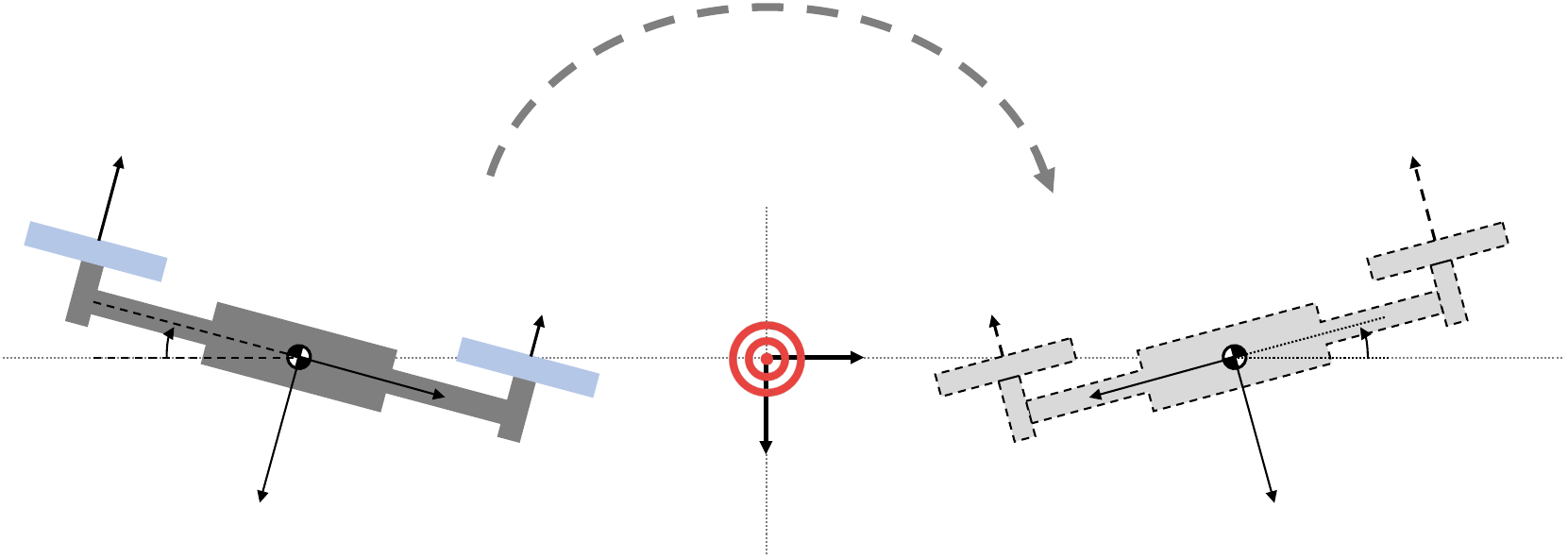}}
		\put(137,25.5){{\scriptsize $\vec e_2$}}
		\put(116,16){{\scriptsize $\vec e_3$}}
		\put(19,34.5){{\scriptsize $\phi$}}
		\put(17,67){{\scriptsize $T_4$}}
		\put(227,68){{\scriptsize $T_4$}}
		\put(155,42){{\scriptsize $T_2$}}
		\put(86,42){{\scriptsize $T_2$}}
		\put(59,49){{\small $(s, a)$}}
		\put(160,49){{\footnotesize $(G_\mathcal{S}(s), G_\mathcal{A}(a))$}}
		\put(69,17){{\scriptsize $\vec b_2$}}
		\put(36,4){{\scriptsize $\vec b_3$}}
		\put(178,17){{\scriptsize $\vec b_2$}}
		\put(210,4){{\scriptsize $\vec b_3$}}
		\end{picture}
		\caption{A control problem of planar quadrotors}
		\label{fig:2D_quad}
	\end{figure}

    To illustrate the key idea, we first consider a planar quadrotor shown in \Cref{fig:2D_quad}, which is confined to the plane spanned by $\vec e_2$ and $\vec e_3$. 
    Here, the state and action of the quadrotor (on the left) are given by $s = (x_2, x_3, v_2, v_3, \phi, \dot{\phi}) \in \mathbb{R}^{6}$ and $a = (T_2, T_4)\in \mathbb{R}^2$, where $\phi$ and $\dot{\phi}$ are the roll angle and the angular velocity, respectively.
    If we flip the quadrotor horizontally to the right side, more precisely by the state transform $G_\mathcal{S}(s) = (-x_2, x_3, -v_2, v_3, -\phi, -\dot\phi)$ and the action transform $G_\mathcal{A} (a) = (T_2,T_4)$,  
    then both systems exhibit the same dynamic characteristics, 
    except that they are on the opposite side. 
    As such, the optimal action on the left can be transformed to the right side, and vice versa, thereby reducing the domain of the state space to be considered for optimization into half. 
    In other words, the above planar quadrotor is symmetric with respect to a horizontal flip, and the symmetry yields certain equivariance properties in the value function and the optimal action, to be exploited for improving sample efficiency.
    This idea is more formally developed as follows.

	\subsection{$\Sph^1$--Symmetry of Quadrotor} 
	
    More generally, we show that the quadrotor dynamics presented in \Cref{sec:PF} is symmetric with respect to the rotation about the vertical axis. 
    Specifically, consider the following group action of $\Sph^1=\{ q \in\Re^2 \,|\, \|q\|=1\}$  parameterized by $\theta\in(-\pi, \pi])$~\footnote{To avoid any confusion with the action $a$ of MDP, this is referred to as \textit{group} action.}:
    \begin{align}
        G_\mathcal{S} (s;\theta) & = (\exp(\theta\hat e_3) x, \exp(\theta\hat e_3) v, \exp(\theta\hat e_3) R, \Omega),\label{eqn:G_S} \\
        G_\mathcal{A} (a) & = a.\label{eqn:G_A}
    \end{align}
    In other words, the group action $g_\theta(s,a) = (G_\mathcal{S}(s;\theta),G_\mathcal{A}(a))$ corresponds to the rotation of the complete system by the angle $\theta$ about $\vec e_3$.
    In \eqref{eqn:G_S} and \eqref{eqn:G_A}, the angular velocity $\Omega$ and the action $a$ seem to be unchanged. 
    However, this is because both $\Omega$ and $a$ are resolved in the body-fixed frame. 
    For example, the angular velocity is rotated from $\omega = R \Omega$ into $\exp(\theta\hat e_3)R\Omega = \exp(\theta\hat e_3) \omega$ by the group action when perceived with respect to the inertial frame. 
    From now on, the group actions on the state and the action are denoted by the same symbol $g_\theta$, i.e., $G_\mathcal{S}(s;\theta) = g_\theta s$ and $G_\mathcal{A}(a) = g_\theta a$.

	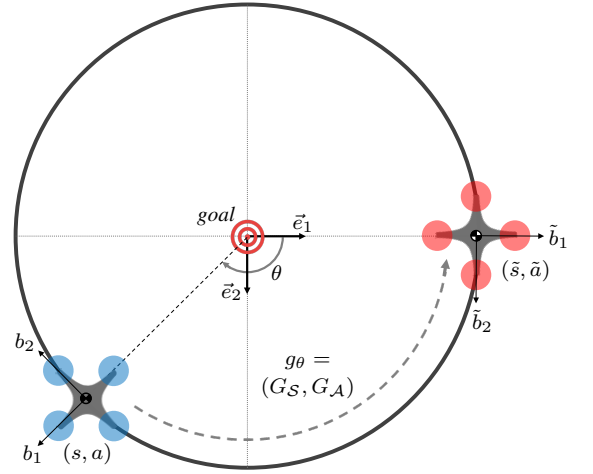
\begin{figure}
		\centering
        \begin{picture}(175,185)
            \footnotesize
            \put(0,0){\includegraphics[scale=0.37]{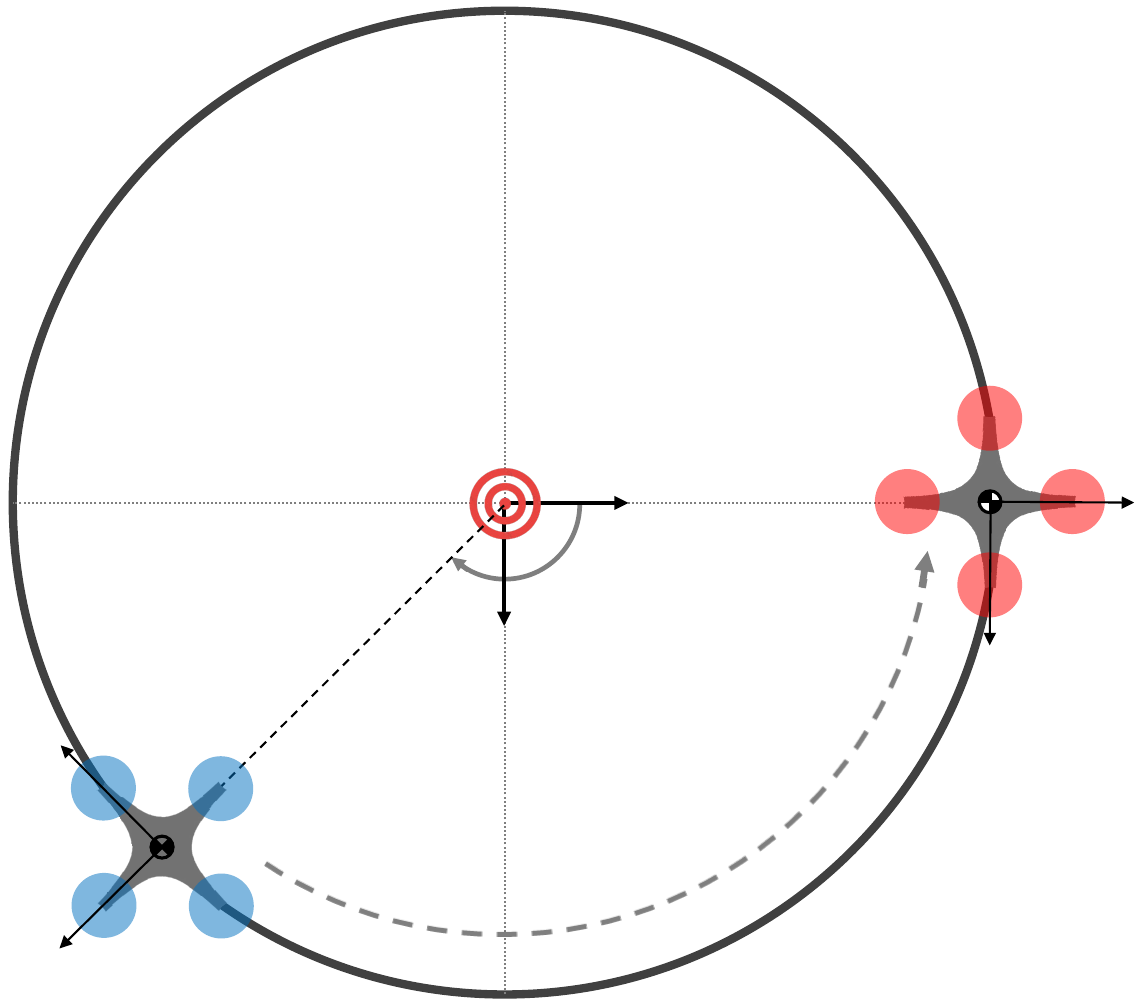}}
            \put(20,5){{$(s, a)$}}
            \put(183,74){{ $(\tilde s, \tilde a)$}}
            \put(95,30){{\shortstack[c]{$g_\theta =$ \\ $(G_\mathcal{S}, G_\mathcal{A})$}}}
			\put(107,92){{$\vec e_1$}}
			\put(80,66){{$\vec e_2$}}
			\put(99,72.5){{$\theta$}}
			\put(71,95){\textit{goal}}
			\put(5,4.5){{$b_1$}}
			\put(2,47){{$b_2$}}
			\put(204,86){{$\tilde{b}_1$}}
			\put(175,54){{$\tilde{b}_2$}}
        \end{picture}
        \caption{Illustration of the group action corresponding to the rotation about $\vec e_3$}
		\label{fig:rot}
	\end{figure}

    Next, we show that the quadrotor dynamics is symmetric with respect to the group action $g_\theta$.
	\begin{prop}\label{prop:sym}
        Let the equations of motion \eqref{eqn:x_dot}--\eqref{eqn:W_dot} be consolidated into
        \begin{align}
            \dot s = F(s,a),\label{eqn:s_dot}
        \end{align}
        for $F:\mathcal{S}\times \mathcal{A}\rightarrow T\mathcal{S}$, where $T\mathcal{S}$ denotes the tangent bundle of the state space.  
        Also, let $(\tilde s, \tilde a) = g_\theta(s,a)$ for the group action given by \eqref{eqn:G_S} and \eqref{eqn:G_A}.
        Then, $(\tilde s, \tilde a)$ also satisfies the equations of motion, i.e., 
        \begin{align}
            \dot{\tilde s} = F(\tilde s, \tilde a).
        \end{align}
        Or equivalently, $F\circ g_\theta = g_\theta \circ F$, considering that the group action is extended to $T\mathcal{S}$.  
        This corresponds to the equivariance of $F$ with respect to $g_\theta$, and it is also stated that the dynamics is symmetric about $g_\theta$.
	\end{prop}
	
	\begin{proof}
		Let $g_\theta(x,v,R,\Omega ) = (\tilde x, \tilde v, \tilde R, \tilde \Omega)$.
		First, we have
		\begin{align*}
		\dot{\tilde x} = \exp(\theta\hat e_3) v = \tilde v.
		\end{align*}
		Similarly, we have $m \dot{\tilde v} = m\exp(\theta\hat e_3) \dot v$.
		Substituting \eqref{eqn:v_dot} into the above, 
		\begin{align*}
		m \dot{\tilde v} & = \exp(\theta\hat e_3) (mg e_3 - f R e_3) \\
		& = mg e_3 - f\tilde R e_3,
		\end{align*}
		as $\exp(\theta\hat e_3)e_3 = e_3$ for any $\theta$.
		These are equivalent to the translational dynamics \eqref{eqn:x_dot} and \eqref{eqn:v_dot}.
		
        Next, for \eqref{eqn:R_dot},
		\begin{align*}
		\dot{\tilde R} = \exp(\theta\hat e_3)\dot R = \exp(\theta\hat e_3) R \hat\Omega = \tilde{R}\hat{\tilde\Omega}.
		\end{align*}
		Also, as $\tilde\Omega=\Omega$, it trivially satisfies \eqref{eqn:W_dot}.
	\end{proof}

    This implies that when the state-action trajectory $(s(t),a(t))$ for $t\in[0,T]$ is the solution of \eqref{eqn:x_dot}--\eqref{eqn:W_dot} representing the quadrotor dynamics, 
    then the rotated trajectory $g_\theta(s(t),a(t))$ also satisfies \eqref{eqn:x_dot}--\eqref{eqn:W_dot} for any $\theta\in\Sph^1$.
    This is not surprising as the direction of $\vec e_1$ (or $\vec e_2$) in the horizontal plane is completely arbitrary in the formulation of the quadrotor dynamics. 

    This further allows us to define an equivalent class for the state-action trajectories over an interval $[0,T]$. 
    Let 
    \begin{align}
        (s(t),a(t))\sim(\tilde s(t), \tilde a(t)) 
    \end{align}
    if there exists $\theta\in\Sph^1$ such that $(\tilde s(t),\tilde a(t)) = g_\theta(s(t), a(t))$ for all $t\in[0,T]$. 
    It is straightforward to show the reflectivity, symmetry, and transitivity of the above binary relation, to verify that it is an equivalent relation~\cite{kelley2017general}. 
    For $(s,a)\in\mathcal{S}\times\mathcal{A}$, define its equivalent class as
    \begin{align}
        [s,a] = \{ (\tilde s, \tilde a) \in\mathcal{S}\times \mathcal{A} \,| \,(s,a)\sim(\tilde s, \tilde a)\}.
    \end{align}
    Then, the quadrotor dynamics can be characterized completely on the quotient space of $\mathcal{S}\times\mathcal{A}$ by the equivalence relation $\sim$, denoted by $\mathcal{S} \times \mathcal{A} /$ $\sim$.

	\subsection{Equivariant Reinforcement Learning}

    Next, we show how the symmetry properties can be exploited in the reinforcement learning. 
    It has been shown that if the reward also satisfies the symmetry property, the corresponding value function is invariant and the optimal policy is equivariant in a discrete-time setting~\cite{wang2022mathrm}.
    Here we establish the correspondent results for deterministic, continuous-time dynamics by extending the continuous reinforcement learning formulated in~\cite{doya2000reinforcement,munos2000study}.
    This is to formulate an equivariant reinforcement learning framework for the inherent, continuous-time quadrotor dynamics, without resorting to any discretization scheme.

    For $0<\gamma<1$, let the value function of a policy $\pi:\mathcal{S}\rightarrow\mathcal{A}$ be defined by
    \begin{align}
        V_\pi (t,s(t)) = \int_t^\infty \gamma^{\tau} r(s(\tau),a(\tau)) d\tau,\label{eqn:V}
    \end{align}
    where the action at any time is defined by the policy as $a(t)=\pi(s(t))$. 
    The objective is to construct the optimal policy $\pi^*(s(t))$ maximizing $V_\pi(t,s(t))$.
    When the dynamics and the reward are symmetric with respect to a group action as presented in \Cref{prop:sym}, the value function and the optimal policy satisfy the following properties. 

    \begin{prop}
        Consider a continuous-time MDP to maximize the value function \eqref{eqn:V} under the dynamics \eqref{eqn:s_dot}.
        Suppose that $F:\mathcal{S}\times\mathcal{A}\rightarrow T\mathcal{S}$ is equivariant and $r:\mathcal{S}\times\mathcal{A}\rightarrow\Re$ is invariant with respect to a group action $g:\mathcal{S}\times\mathcal{A}\rightarrow\mathcal{S}\times\mathcal{A}$, i.e., 
        \begin{gather}
            F\circ g = g\circ F, \label{eqn:F_eqv} \\
            r\circ g = r. \label{eqn:r_inv}
        \end{gather}
        For a given policy $\pi(s):\mathcal{S}\rightarrow\mathcal{A}$, define a new policy induced by $g$ as
        \begin{align}
            \tilde \pi (\tilde s) = g \pi(g^{-1}\tilde s).
        \end{align}
        Then, the following properties hold:
        \renewcommand{\theenumi}{(\roman{enumi}}
        \begin{enumerate}
            \item The value function is invariant under the group action, i.e., $V_\pi = V_{\tilde\pi}\circ g$. 
            \item The optimal policy is equivariant under the group action, i.e., $\pi^* \circ g = g\circ \pi^*$. 
        \end{enumerate}
    \end{prop}
    \begin{proof}
        Let $(s(t),a(t))$ be a trajectory of \eqref{eqn:s_dot}, driven by a policy $\pi(s)$. 
        Then, according to \Cref{prop:sym}, $(\tilde s(t), \tilde a(t))=(gs(t), ga(t))$ is another trajectory of \eqref{eqn:s_dot}, where
        \begin{align*}
            \tilde a = g a = g \pi(s) = g \pi(g^{-1}\tilde s) = \tilde \pi(\tilde s).
        \end{align*}
        Thus, $(\tilde s(t), \tilde a(t))$ is a trajectory of \eqref{eqn:s_dot} from the transformed policy $\tilde \pi$. 
        We have
        \begin{align*}
            V_{\tilde \pi}(t, g s(t)) & = \int_t^\infty \gamma^\tau r(\tilde s(\tau),\tilde a(\tau)) d\tau \\
                                   & = \int_t^\infty \gamma^\tau r(g s(\tau),g a(\tau)) d\tau,
        \end{align*}
        which reduces to \eqref{eqn:V} from \eqref{eqn:r_inv}.
        This shows (i).

        Next, split the domain of the integration in \eqref{eqn:V} into two parts such that
    \begin{align*}
        V_\pi(t,s(t)) & = \int_{t}^{t+\Delta t} \gamma^{\tau} r(s(\tau),a(\tau)) d\tau \\
          &\quad + \int_{t+\Delta t}^{\infty} \gamma^{\tau} r(s(\tau),a(\tau)) d\tau\\
          & = r(s(t),a(t))\Delta t + V_\pi(t+\Delta t, s(t+\Delta t)) + o(\Delta t).
    \end{align*}
    Here, $V_\pi(t+\Delta t, s(t+\Delta t))$ is expanded into
    \begin{align*}
        V_\pi(t,s(t)) + \deriv{V_\pi}{t}\Delta t + \deriv{V_\pi}{s}F(s(t),a(t))\Delta t + o(\Delta t).
    \end{align*}
    Substituting this into the above and rearranging
    \begin{align}
        -\deriv{V_\pi}{t} = r(s(t),a(t)) + \deriv{V_\pi}{s}F(s(t),a(t)).\label{eqn:HJB0}
    \end{align}
    as $\Delta t\rightarrow 0$.
    The left-hand side is rewritten as follows.
    The value function can be reorganized into
    \begin{align*}
        V_\pi(t,s(t)) = \gamma^t \int_t^\infty \gamma^{\tau-t} r(s(\tau),a(\tau)) d\tau,
    \end{align*}
    where the integral at the second factor is independent of $t$. 
    Thus, its derivative with respect to $t$ is given by
    \begin{align*}
        \deriv{V_\pi}{t} & =  \gamma^t \log\gamma \int_t^\infty \gamma^{\tau-t} r(s(\tau), a(\tau)) d\tau\\
        & = \log\gamma V_\pi(t,s(t)).
    \end{align*}
    Substituting this back to \eqref{eqn:HJB0}, 
    \begin{align*}
        \log\gamma^{-1} V_\pi(t,s(t)) = r(s(t),a(t)) + \deriv{V_\pi(t,s(t))}{s}F(s(t), a(t)).
    \end{align*}

    Let $V^*$ be the optimal value function, obtained by the optimal policy $\pi^*$.
    The above yields the following Hamilton-Jacobi-Bellman equation:
    \begin{align*}
        \log\gamma^{-1} V^*(t,s(t)) = \max_{a\in\mathcal{A}} \left\{ r(s,a) + \deriv{V^*(t,s(t))}{s}F(s, a) \right\},
    \end{align*}
    and the optimal policy is given by
    \begin{align}
        \pi^*(s(t)) = \arg \max_{a\in\mathcal{A}} \left\{ r(s,a) + \deriv{V^*(t,s(t))}{s}F(s, a) \right\}.\label{eqn:pi_opt}
    \end{align}

    Therefore, the optimal action at $\tilde s = gs$ is given by
    \begin{align*}
        (\pi^*\circ g) (s) = \pi^*(\tilde s) = \arg \max_{\tilde a\in\mathcal{A}} \left\{ r(\tilde s,\tilde a) + \deriv{V^*(t,\tilde s)}{\tilde s}F(\tilde s, \tilde a) \right\}.
    \end{align*}
    From \eqref{eqn:F_eqv} and \eqref{eqn:r_inv}, we have $F(\tilde s,\tilde a) = g F(s, g^{-1}\tilde a)$ and $r(\tilde s, \tilde a) = r(s, g^{-1}\tilde a)$. 
    Further, utilizing the property (i),
    \begin{align*}
        \deriv{V^*(t,\tilde s(t))}{\tilde s} = 
        \deriv{V^*(t,s(t))}{s}\deriv{s}{\tilde s} = \deriv{V^*(t,s(t))}{s}g^{-1}
    \end{align*}
    Substituting these,
    \begin{align}
        \pi^*(\tilde s(t)) = \arg \max_{\tilde a\in\mathcal{A}} \left\{ r(s,g^{-1}\tilde a) + \deriv{V^*(t,s(t))}{s}F(s, g^{-1}\tilde a) \right\},\label{eqn:pis_opt}
    \end{align}
    which is equivalent to \eqref{eqn:pi_opt}, except that $a$ in \eqref{eqn:pi_opt} is replaced by $g^{-1}\tilde a$.
    As such, if $a^*(t)=\pi^*(s(t))$ is the optimal action of \eqref{eqn:pi_opt}, then the right hand side of \eqref{eqn:pis_opt} is maximized when $g^{-1}\tilde a = a^*$, or equivalently $\tilde a = g a^*$. 
    This shows that $\pi^* (gs) = g\pi^*(s)$. 
\end{proof}

    Equivariant reinforcement learning is to utilize the invariance of the value function and the equivariant of the optimal action with respect to the symmetry of the dynamics and the reward. 
    For the quadrotor dynamics, the symmetry of the dynamics \eqref{eqn:F_eqv} is established by \Cref{prop:sym}, and it is straightforward to show the invariance of the reward \eqref{eqn:r_inv} as the operation of taking norm of a vector in \eqref{eqn:reward} is not affected by any rotation.
    Further, as the group action $g_\theta$ on the action $a$, given by \eqref{eqn:G_A} is the identity map, the optimal action is in fact invariant as well. 
    In short, for the quadrotor dynamics, we have
    \begin{align}
        V(s) = V(\tilde s),\quad \pi^*(s) = \pi^*(\tilde s),\label{eqn:quadRL_eqv}
    \end{align}
    for any $\tilde s\in[s]$. 

    In deep reinforcement learning, the value function and the policy are represented by deep neural networks, which are trained by a set of trajectories.
    As such, successful implementation of reinforcement learning often requires a massive number of trajectories. 
    Here, the sampling efficiency of reinforcement learning can be substantially improved by exploiting \eqref{eqn:quadRL_eqv} as presented below.

    \section{Numerical Experiments}\label{sec:NE}
	This section presents a neural network structure that respects the above equivariance property, to be utilized in the proposed equivariant reinforcement learning. 
    Then, we show the detailed implementation and simulation environments for benchmark studies. 
	
	\subsection{Neural Network Structures}
	
    In reinforcement learning, neural networks that approximate the optimal policy and the value function are referred to as \textit{actor} and \textit{critic} networks, respectively. 
    To impose the invariant properties \eqref{eqn:quadRL_eqv} in the actor and the critic, one can develop group equivariant neural networks with respect to the group action of \eqref{eqn:G_S} and \eqref{eqn:G_A}. 
    However, the most of the existing results in equivariant neural networks deal with inputs of images or discrete actions, and they are not suitable for the presented rotational action on vectors and matrices. 

    Instead, from \eqref{eqn:quadRL_eqv}, the approximation of the value function and the policy is performed on the quotient space $\mathcal{S}\times\mathcal{A}/\sim$. 
    For the equivalent class $[s]$ of any $s\in\mathcal{S}$, we choose a representative element $\tilde s$ by
    \begin{align}
        [s] =  g_{[\theta]} s,
    \end{align}
where $[\theta] = -\mathrm{atan2}(x_2, x_1)$.
    Using this, the equivalent class can be identified with the above representative element, such that the quotient space is considered as a set of the representative elements.
    Thus, 
    \begin{align}
        V(s) = V([s]),\quad \pi^*(s)=\pi^*([s]).\label{eqn:quadRL_eqv_rep}
    \end{align}
    As the group action is one-dimensional, this reduces the domain of the value and the policy by one. 
    The particular choice of the rotation angle $[\theta]$ ensures that the second element of the position is always zero, i.e., $[x]_2=0$ (see \Cref{fig:rot}).
    As such, it can be simply dropped from the input of the actor and the critic. 
    Specifically, the input to the each network is $([x]_1, [x]_3, [v], [R], [\Omega]) \in \mathbb{R}^{17}$.
	
	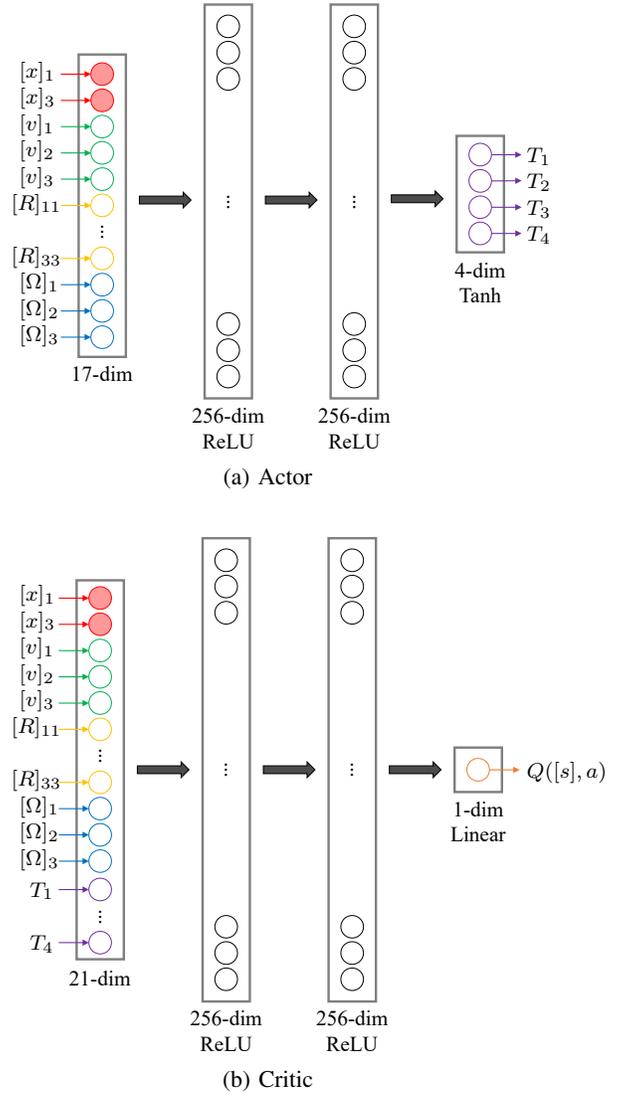
\begin{figure}
		\centering
		\begin{tabular}{ c }
			\begin{picture}(160,175)
				\footnotesize
				\put(0,0){\includegraphics[scale=0.4]{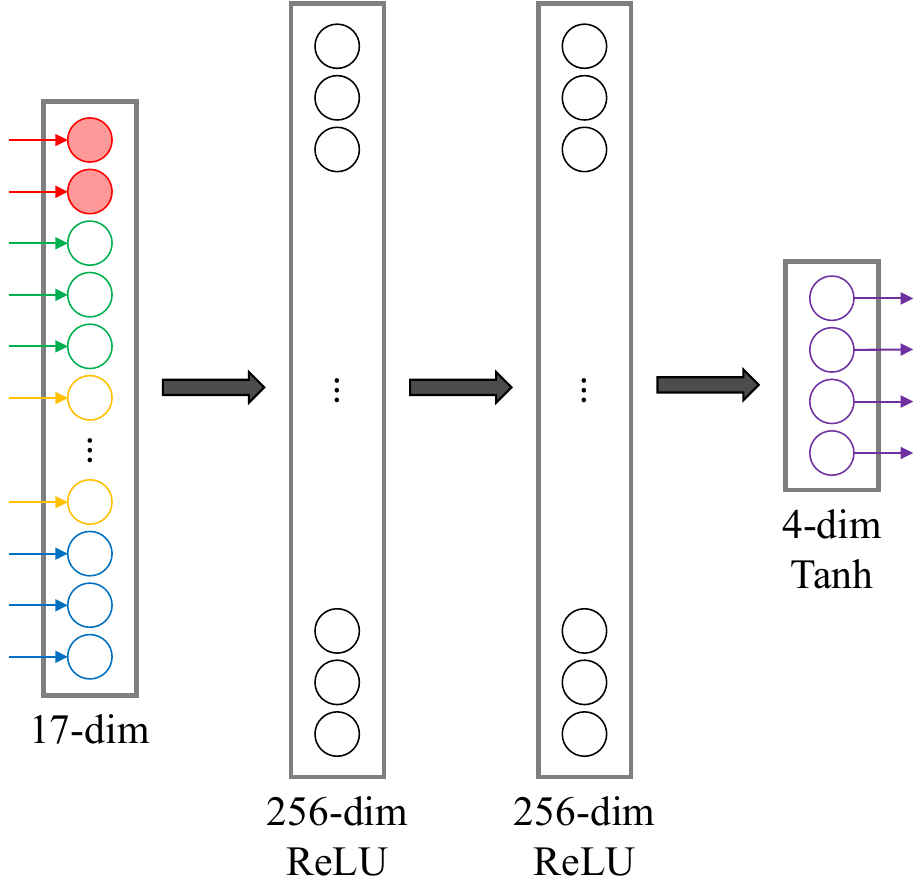}}
				\put(-14,142){{$[x]_1$}}
				\put(-14,132){{$[x]_3$}}
				\put(-14,122){{$[v]_1$}}
				\put(-14,112){{$[v]_2$}}
				\put(-14,102){{$[v]_3$}}
				\put(-17,92){{$[R]_{11}$}}
				\put(-17,72){{$[R]_{33}$}}
				\put(-14,62){{$[\Omega]_1$}}
				\put(-14,52){{$[\Omega]_2$}}
				\put(-14,42){{$[\Omega]_3$}}
				
				\put(178,110){{$T_1$}}
				\put(178,100){{$T_2$}}
				\put(178,90){{$T_3$}}
				\put(178,80){{$T_4$}}
			\end{picture}\\
			\small (a) Actor \\
			\\[0.5pt]
			\begin{picture}(160,200)
				\footnotesize
				\put(0,0){\includegraphics[scale=0.4]{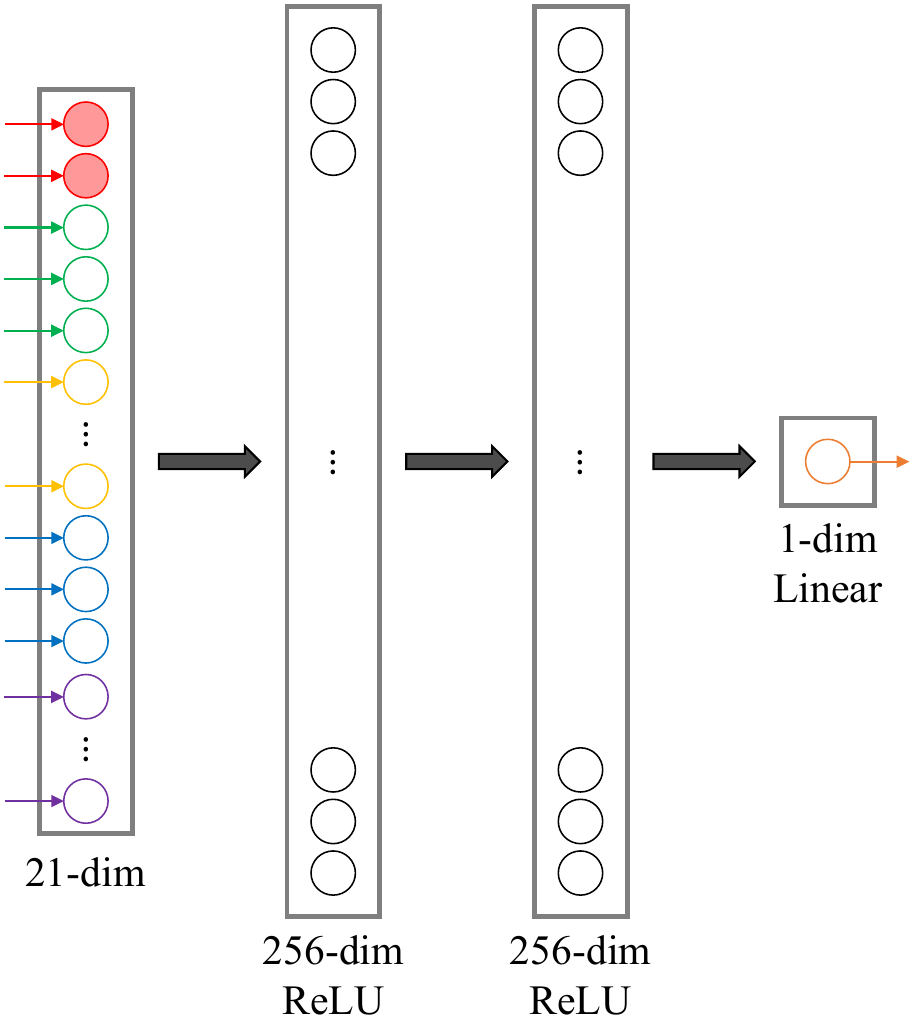}}
				\put(-14,172){{$[x]_1$}}
				\put(-14,162){{$[x]_3$}}
				\put(-14,152){{$[v]_1$}}
				\put(-14,142){{$[v]_2$}}
				\put(-14,132){{$[v]_3$}}
				\put(-17,122){{$[R]_{11}$}}
				\put(-17,102){{$[R]_{33}$}}
				\put(-14,92){{$[\Omega]_1$}}
				\put(-14,82){{$[\Omega]_2$}}
				\put(-14,72){{$[\Omega]_3$}}
				\put(-10,60){{$T_1$}}
				\put(-10,40){{$T_4$}}
				
				\put(178,105){{$Q([s], a)$}}
			\end{picture}\\
			\small (b) Critic
		\end{tabular}
		\caption{The architecture of neural networks: (a) Actor and (b) Critic}
		\label{fig:NN}
	\end{figure}

    Note that \eqref{eqn:quadRL_eqv_rep} can be adopted to any reinforcement learning scheme.
    In this paper, we apply it to TD3 and SAC, to obtain two equivariant reinforcement learning schemes, referred to as equivariant TD3 and equivariant SAC, respectively.  
    As they are developed inherently for a discrete-time MDP, the quadrotor equations of motion are discretized by a Runge-Kutta method. 
    Reinforcement learning for a discrete-time MDP often relies on the state-action value function, $Q(s,a)$. 
    As a linear combination of the reward $r(s,a)$ and the value $V(s)$, the state-action value function is also group invariant~\cite{wang2022mathrm}, i.e., $Q\circ g = Q$. 
    Therefore, $Q(s,a) = Q([s],a)$ for the presented quadrotor dynamics. 
	
	The structure of the actor network in this paper is illustrated in \Cref{fig:NN}(a), where the actor maps the reduced quadrotor state to its motor commands. 
	TD3 and SAC share the same inputs and outputs consisting of a 17-dimensional state vector and a 4-dimensional action vector, respectively.
	One major difference is that TD3 is based on the \textit{deterministic} policy actor, whereas the \textit{stochastic} policy actor is updated in SAC.
	A hyperbolic tangent activation function in the output layer ensures a proper range of control signals.
	The critic networks are depicted in \Cref{fig:NN}(b).
    TD3 and SAC take both the state and action as inputs and estimate the action-value function $Q([s],a)$ as outputs.
	All networks are built as multilayer perceptron (MLP) networks with two hidden layers of 256 nodes with \textit{ReLU} activation function.

%

	\subsection{Simulation Environments}
	As shown in \Cref{fig:env}, the simulation environment is developed in Python based on the quadrotor dynamics in Section \ref{sec:DM} with the open-source library OpenAI Gym \cite{brockman2016openai}, which is a popular RL algorithm development toolkit.
	TD3 and SAC algorithms are implemented with a machine learning framework PyTorch \cite{paszke2019pytorch}.
	The training and evaluation were executed for 1.5 million time steps with
	the use of NVIDIA CUDA, on a GPU workstation powered by NVIDIA A100-PCIE-40GB.
	Table \ref{tab:hyper} presents hyperparameters used for training TD3, SAC, and Equivariant RL agents.
	TD3 and SAC share some default parameters, such as learning rate and discount factor, unless explicitly indicated in the table.

	\begin{figure}
		\centering
		\includegraphics[scale=0.245]{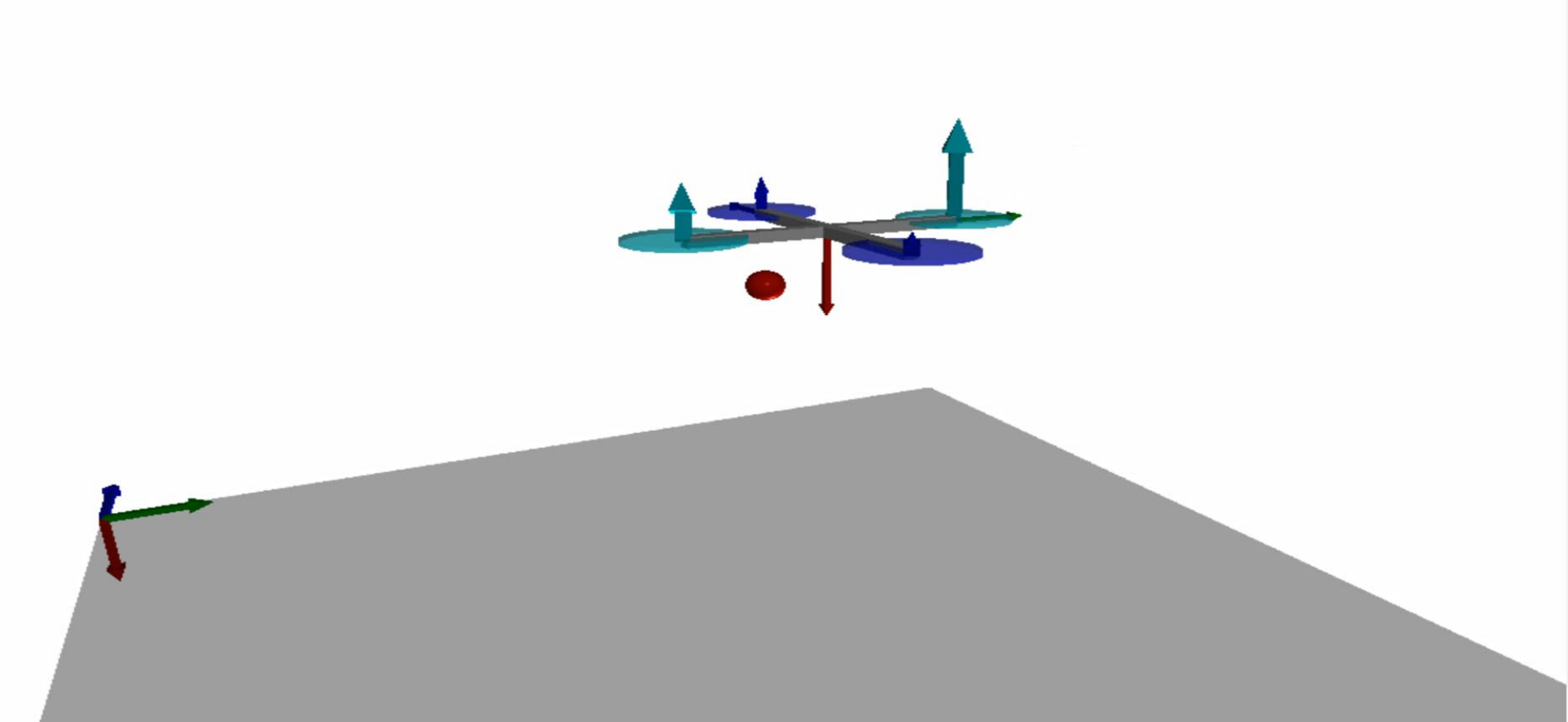}
		\caption{Simulation environment implemented in Python}
		\label{fig:env}
	\end{figure}

	\begin{table}[ht]
		\centering
        \caption{Hyperparameters used in benchmark studies}
		\label{tab:hyper}
		\begin{tabular}[t]{lc}
			\toprule
			Parameter & Value\\
			\midrule
			Optimizer & Adam \\
			Actor learning rate & $3\cdot10^{-4}$ \\
			Critic learning rate & $3\cdot10^{-4}$ \\
			Discount factor, $\gamma$ & 0.99 \\
			Replay buffer size & $10^6$ \\
			Batch size & 256 \\
			Target smoothing coefficient, $\tau$ & 0.005 \\
			\midrule
			Twin Delayed DDPG (TD3) & \\
			~ Exploration noise & 0.1 \\
			~ Target policy noise & 0.2 \\
			~ Policy noise clip & 0.5 \\
			~ Target update interval & 2 \\
			\midrule
			Soft Actor-Critic (SAC) & \\
			~ Entropy regularization coefficient, $\alpha$ & Autotuned \\
			~ Target update interval & 1 \\
			\bottomrule
		\end{tabular}
	\end{table}%


	\subsection{Benchmark Results}
	
	We validated the proposed approach comprehensively through numerical simulations to ensure reliability and stability.
	We took the vanilla TD3 and SAC algorithms, which do not consider the equivariance property,  as the baselines for performance comparison.
	All experiments were executed for 6 random seeds, and the average return was reported once every 5,000 steps without action noise for 30 trajectories.
	During training and evaluation, the quadrotor started with the random states of each episode at an arbitrary initial location in a 3m $\times$ 3m $\times$ 3m space.
	The agents were trained without any auxiliary aid technique such as PID controllers or pre-trained policies.
	
	In \eqref{eqn:reward}, the coefficient $c_x$ is set to 2.0, and the penalizing terms are set $c_v = 0.15$, $c_\Omega = 0.2$, and $c_a = 0.03$ to improve stability and achieve smooth control.
	Note that too large penalties prevent the quadrotor from moving toward its target by focusing only on stabilization.
	The reward is normalized into $[0, 1]$ and rescaled by a factor of 0.1 to ensure convergence.
	Finally, the discount factor is selected as $\gamma = 0.99$.
	
    \begin{figure}
		\centering
		\subfigure[TD3 and Equivariant TD3]{
			\includegraphics[width=0.95\columnwidth]{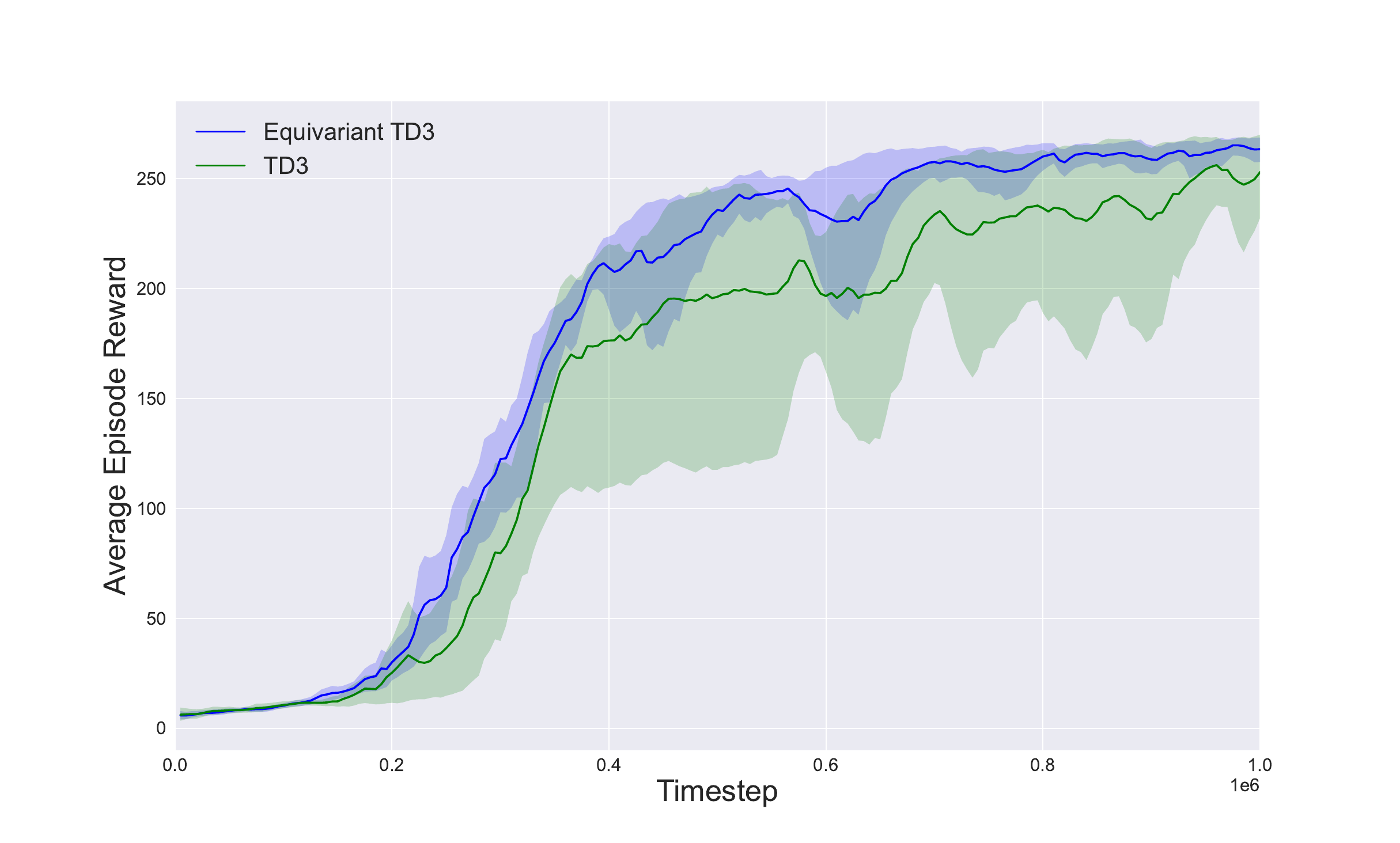} \label{fig:BM_TD3}}
		\\[10pt]
		\subfigure[SAC and Equivariant SAC]{
			\includegraphics[width=0.95\columnwidth]{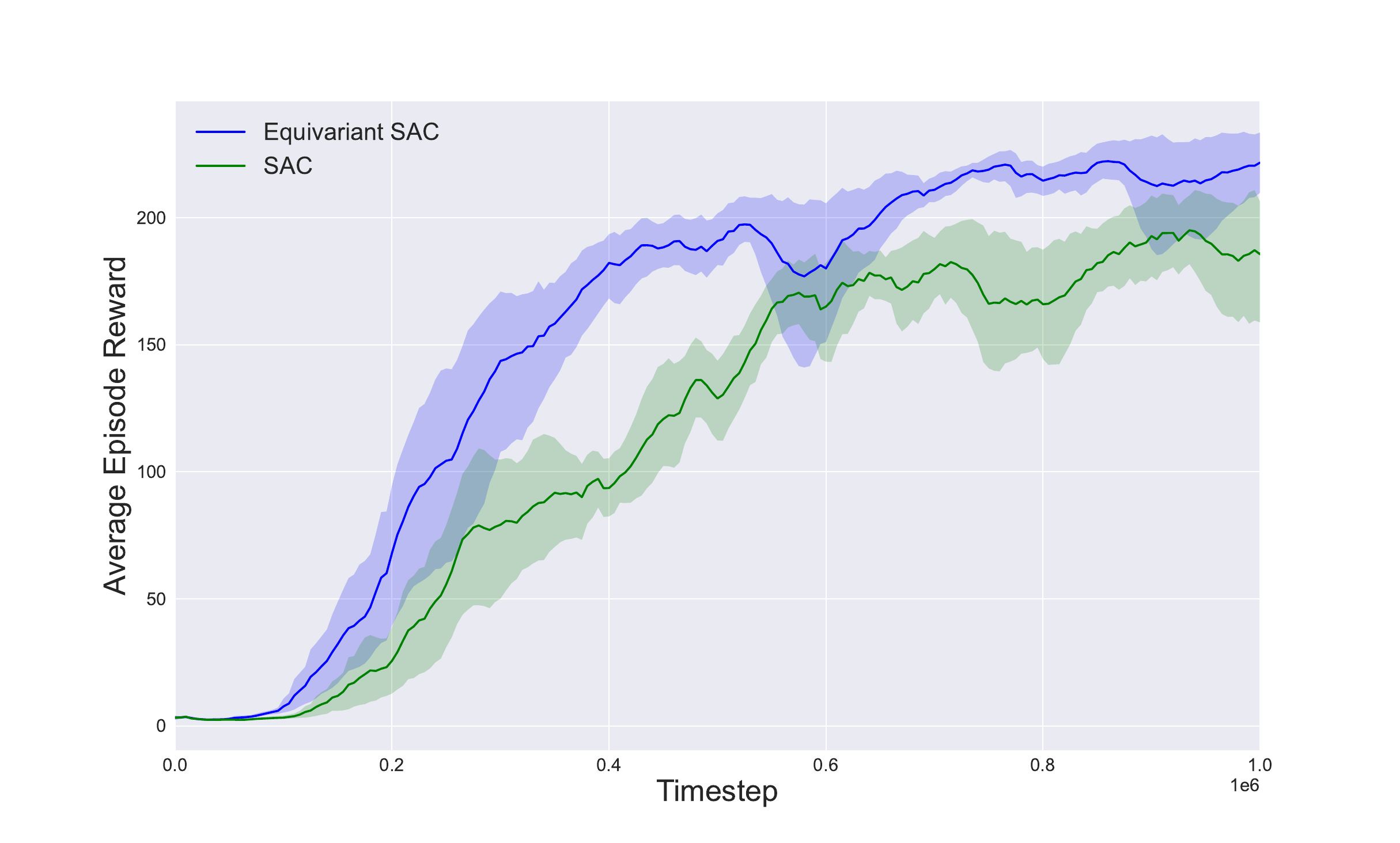} \label{fig:BM_SAC}}
		\caption{Comparison results with respect to the average reward in evaluation: (a) TD3 and (b) SAC}
		\label{fig:BM}
	\end{figure}

	\Cref{fig:BM} shows training curves in terms of the average reward, where the blue curves denote the proposed equivariant methods and the green corresponds to the baselines.
	The solid lines and shaded areas represent the mean value and $2\sigma$ bounds, respectively.
    \Cref{fig:BM_TD3} shows the result of training with TD3 and \Cref{fig:BM_SAC} shows the results of SAC.
	In both cases, the proposed equivariant methods outperform the baselines in terms of the learning speed and convergence.
    In \Cref{fig:BM_TD3}, it is also illustrated that the equivariant TD3 exhibits more consistent results than the vanilla TD3 with respect to the random seed variations. 

	Additionally, the equivariant methods exhibit higher rewards than their counterparts over the same time period.
	In other words, the baselines require more accumulated timesteps to achieve a similar level of reward.
	This confirms that our proposed equivariant framework is more sample-efficient, leading to faster convergence.
		
	Next, to demonstrate the flight performance of the trained policy, the state trajectory and the control inputs for a specific episode are presented in \Cref{fig:traj}, where the top four subfigures show the position error $e_x$, velocity error $e_v$, angular velocity error $e_\Omega$, and motor thrusts $T$. 
    The last subfigure presents the nine elements of attitude $R$.
	As shown in \Cref{fig:traj_ex}, the learned policy successfully controlled the quadrotor from a random position to its target after 6 seconds without any noticeable steady-state error.
    Specifically, the terminal position was $(-0.0028, 0.0017, -0.0012)\mathrm{m}$.
    From \Cref{fig:traj_T}, the thrust commands are reasonable, and they avoid rapid chattering that often appears in RL-based quadrotor controls.
	Thus, the simulation results clearly show that the proposed equivariant framework is not only efficient enough to handle high-dimensional data, but also exhibits desirable properties.

	\begin{figure}[htb]
		\centering
		\begin{tabular}{cc}
			\subfigure[Position error $e_x$]{\includegraphics[width=0.47\columnwidth]{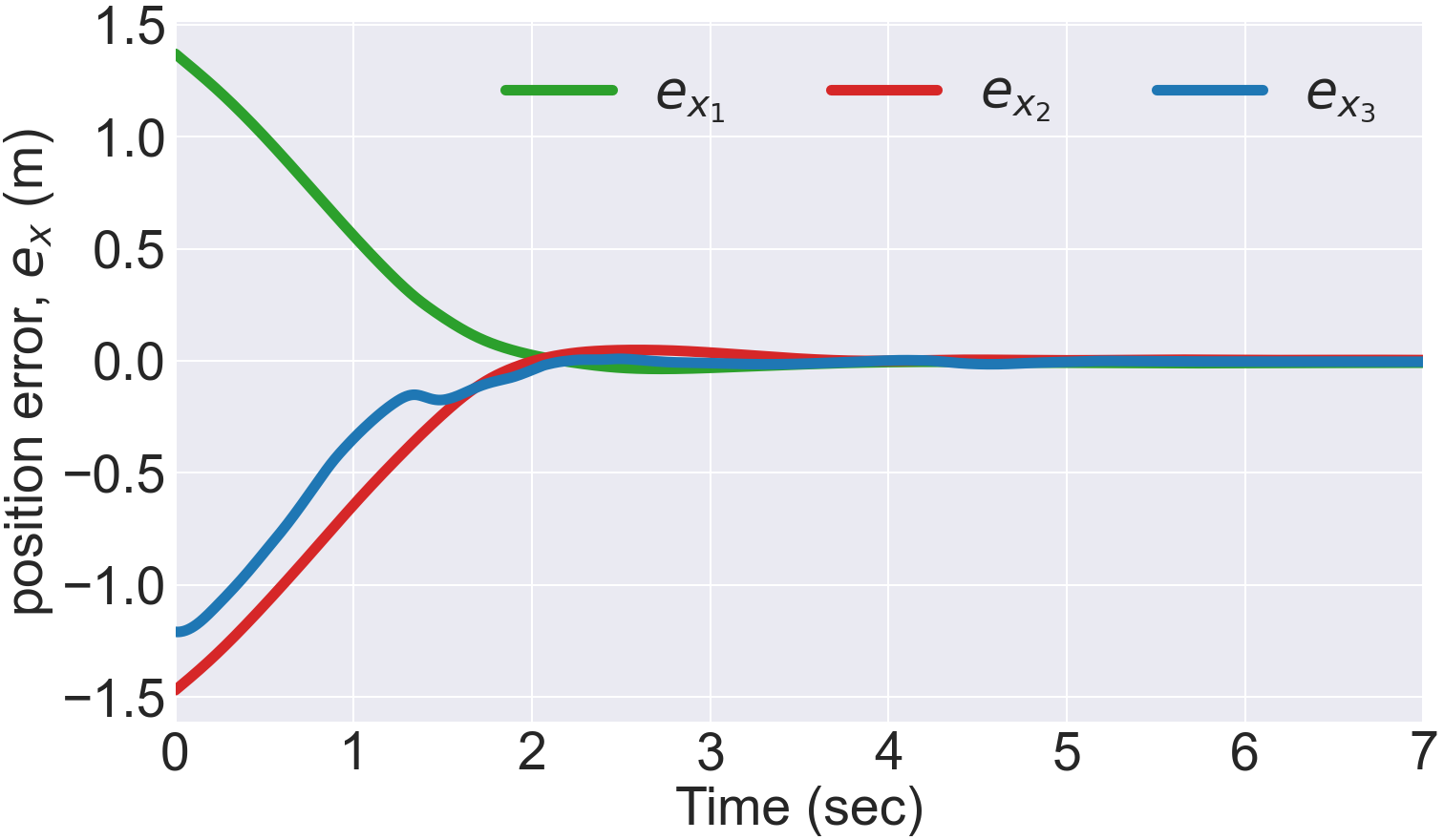} \label{fig:traj_ex}} 
            & \subfigure[Velocity error $e_v$]{\includegraphics[width=0.47\columnwidth]{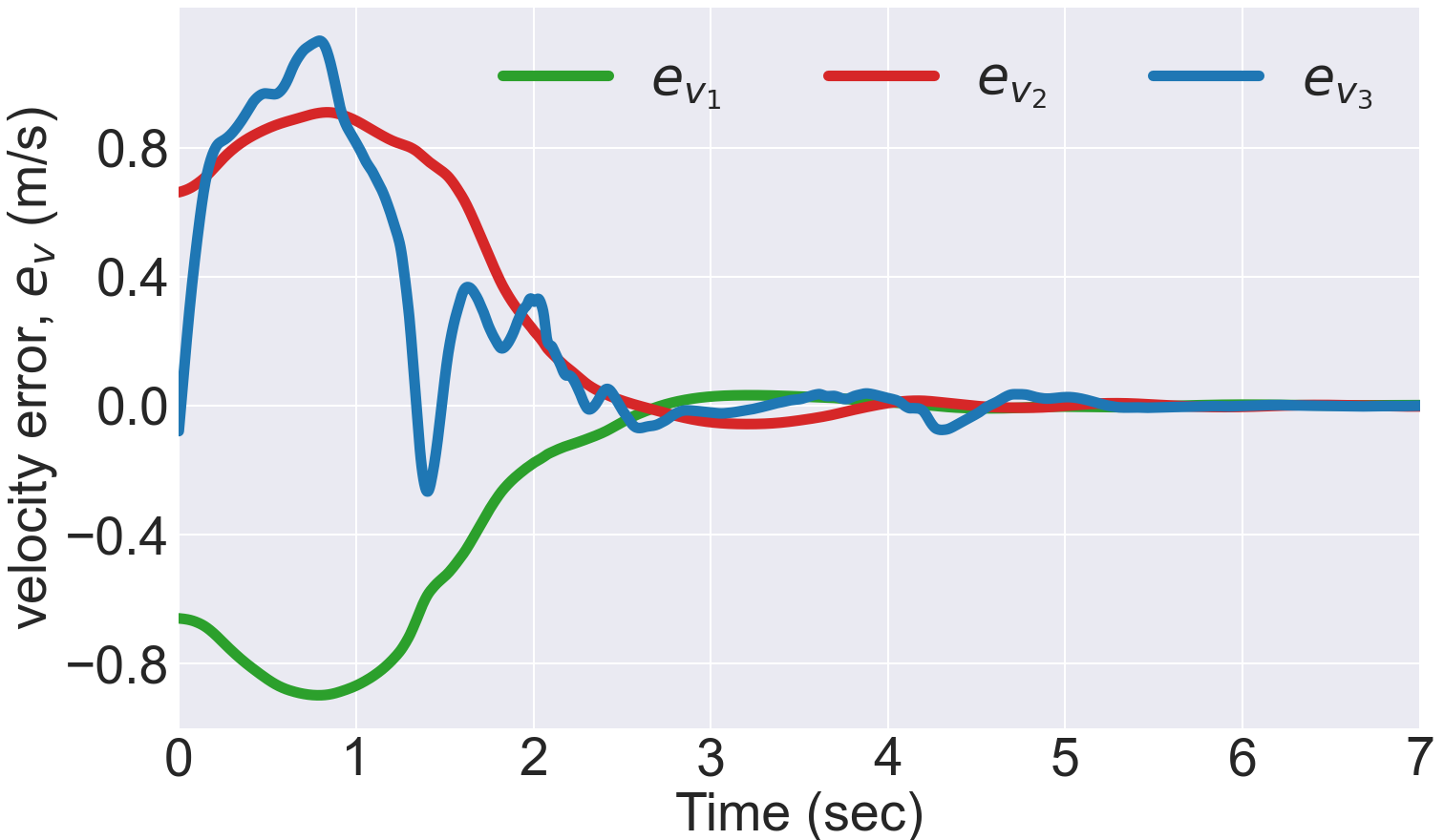} \label{fig:traj_ev}} \\[5pt]
			\subfigure[Angular velocity error $e_\Omega$]{\includegraphics[width=0.47\columnwidth]{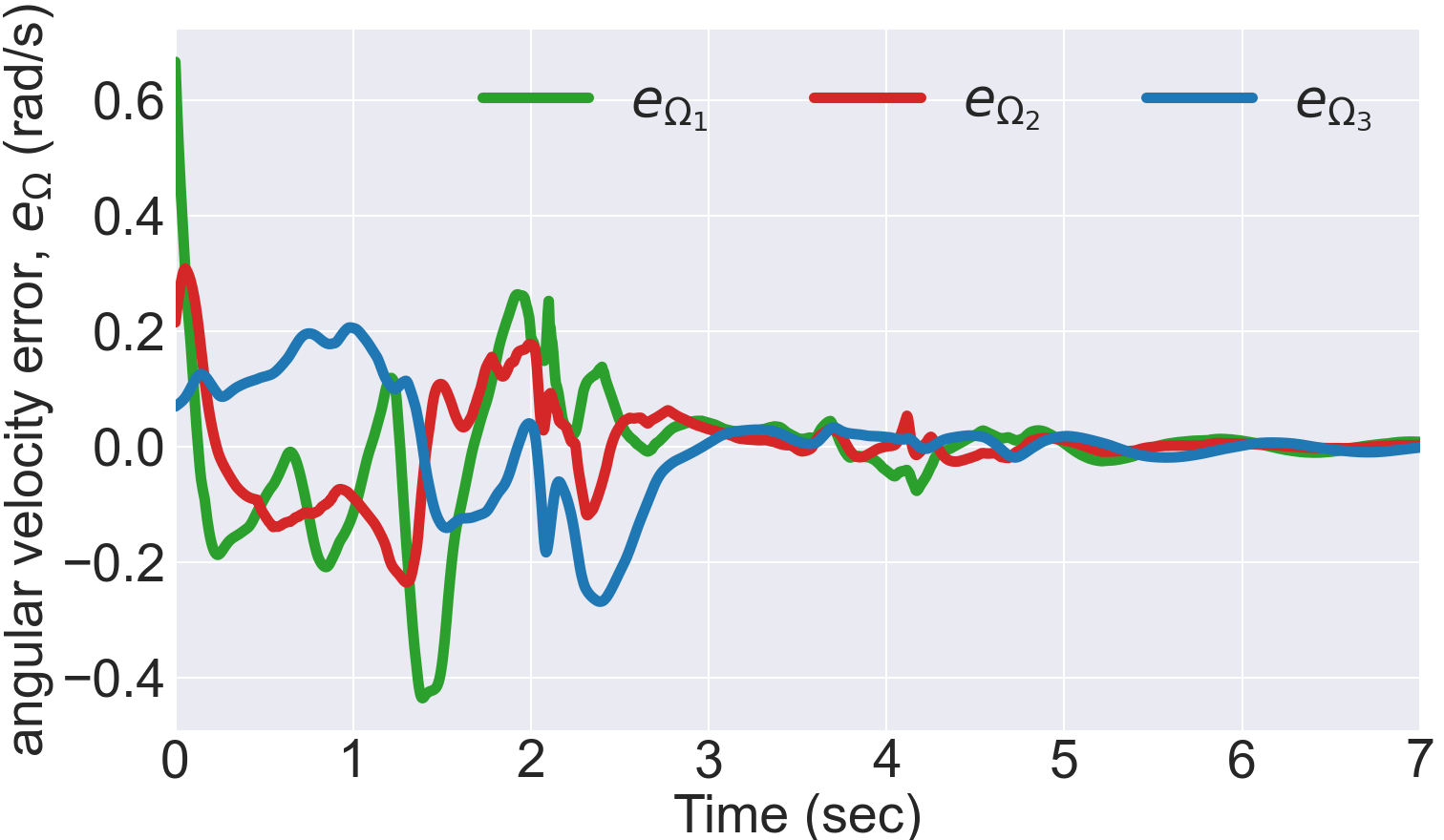} \label{fig:traj_eW}} & \subfigure[Motor thrust ($T_1,T_2,T_3,T_4$)]{\includegraphics[width=0.47\columnwidth]{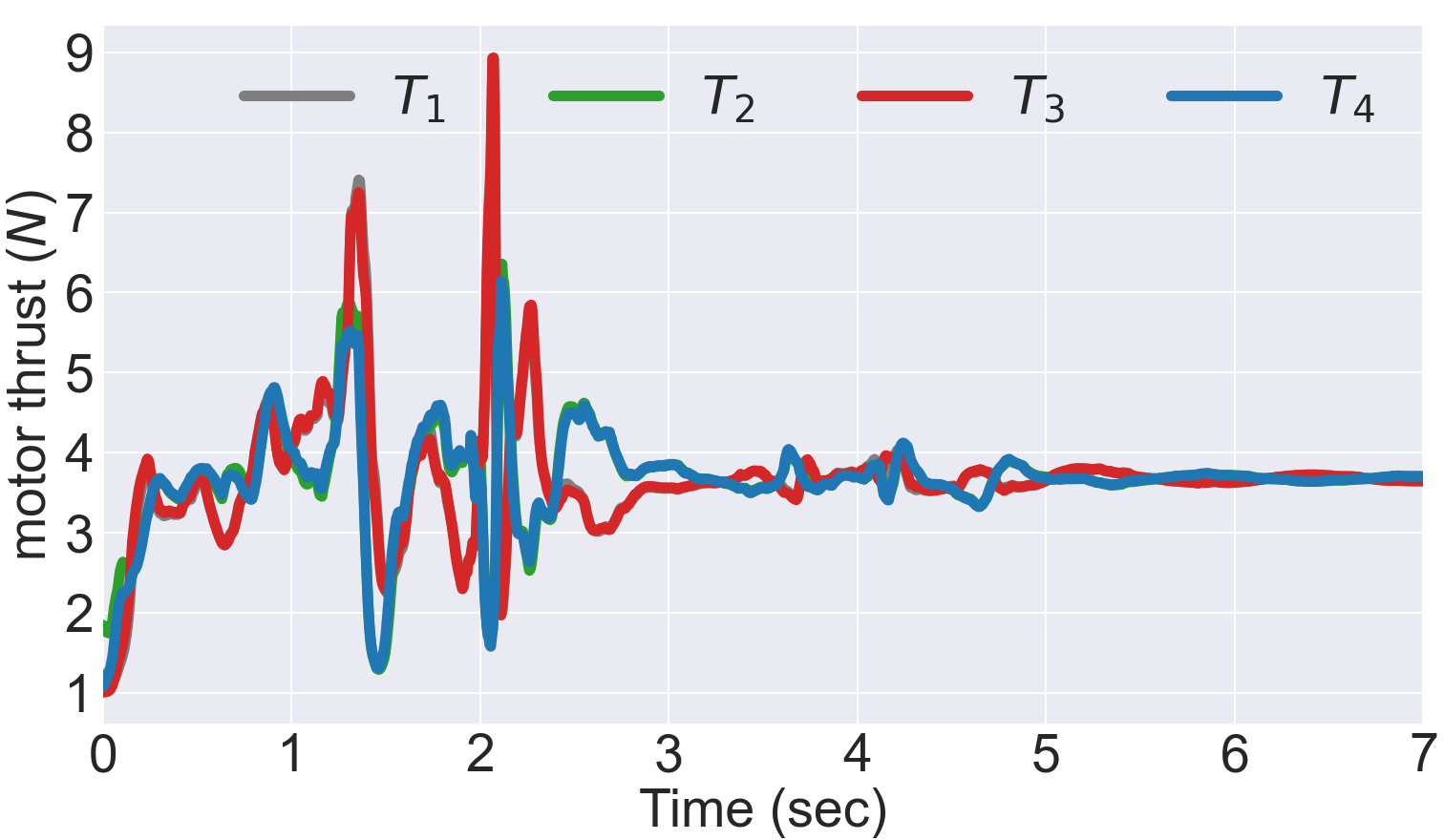} \label{fig:traj_T}} \\[5pt]
			\multicolumn{2}{c}{\subfigure[Attitude $R$]{\includegraphics[width=0.99\columnwidth]{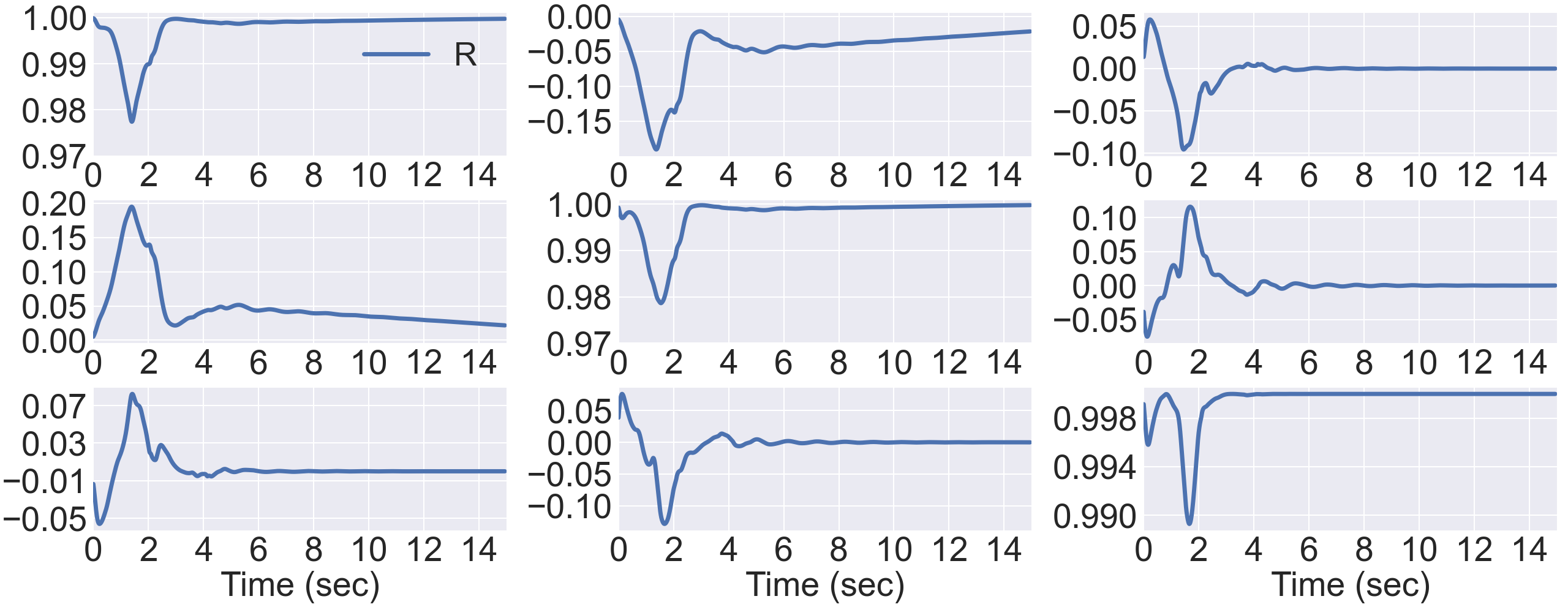} \label{fig:traj_R}}} 
		\end{tabular}
        \caption{Controlled trajectory for a single episode}
		\label{fig:traj}
	\end{figure}





	\section{Conclusion} 
	
	In this paper, we presented a data-efficient reinforcement learning strategy for  quadrotors, by exploiting its symmetry. 
    We identified $\Sph^1$--symmetry of the quadrotor dynamics, which is utilized in the neural network structures for the actor and the critic to reduce the dimension of the input domain by one. 
    Numerical simulation with two popular reinforcement schemes shows that the equivariance property substantially improves sample efficiency, and the trained policy is reasonable. 

	Future work includes incorporating other symmetry properties of the quadrotor.
    For example, since the structure of the quadrotor is symmetrical about its third body-fixed axis, the thrust at the opposite side of the quadrotor can be swapped when it is rotated about the third body-fixed axis by $180^\circ$.
	Also, instead of utilizing the representative element of the equivalent class, the invariance of the value or the policy can be directly imposed via equivariant neural networks. 

	\bibliography{BibTeX}
	\bibliographystyle{IEEEtran}
	
			
\end{document}